\documentclass[final,cleveref]{alt2026} %

\title[Group-realizable multi-group learning by minimizing empirical risk]{Group-realizable multi-group learning by minimizing empirical risk}
\usepackage{times}
\altauthor{%
 \Name{Navid Ardeshir} \Email{navid.ardeshir@columbia.edu}\\
 \addr Columbia University
 \AND
 \Name{Samuel Deng} \Email{samdeng@cs.columbia.edu}\\
 \addr Columbia University
 \AND
 \Name{Daniel Hsu} \Email{djhsu@cs.columbia.edu}\\
 \addr Columbia University
 \AND
 \Name{Jingwen Liu} \Email{jingwenliu@cs.columbia.edu}\\
 \addr Columbia University
}

\Crefformat{equation}{#2(#1)#3}
\crefformat{equation}{#2(#1)#3}
\Crefrangeformat{equation}{#3(#1)#4--#5(#2)#6}
\crefrangeformat{equation}{#3(#1)#4--#5(#2)#6}
\Crefmultiformat{equation}{#2(#1)#3}{ and~#2(#1)#3}{, #2(#1)#3}{ and~#2(#1)#3}
\crefmultiformat{equation}{#2(#1)#3}{ and~#2(#1)#3}{, #2(#1)#3}{ and~#2(#1)#3}

\usepackage{csquotes}
\MakeOuterQuote{"}

\usepackage{mathtools}
\usepackage{dsfont}
\mathtoolsset{centercolon}
\DeclarePairedDelimiterXPP\ind[1]{\mathds{1}}{\lbrace}{\rbrace}{}{#1} %
\DeclarePairedDelimiterX\eval[1]{\lbrace}{\rvert}{#1 \delimsize\rbrace} %
\DeclarePairedDelimiter\card{\lvert}{\rvert} %
\DeclarePairedDelimiter\del{\lparen}{\rparen} %
\DeclarePairedDelimiter\sbr{\lbrack}{\rbrack} %
\let\set\relax
\DeclarePairedDelimiter\set{\lbrace}{\rbrace} %
\DeclarePairedDelimiter\intoo{\lparen}{\rparen} %
\providecommand\given{}
\newcommand\SetSymbol[1][]{%
  \nonscript\:#1\vert
  \nonscript\:
  \mathopen{}}
\DeclarePairedDelimiterX\Set[1]\lbrace\rbrace{%
  \renewcommand\given{\SetSymbol[\delimsize]}
  #1
}
\DeclarePairedDelimiterX\Sbr[1]\lbrack\rbrack{%
  \renewcommand\given{\SetSymbol[\delimsize]}
  #1
}
\DeclarePairedDelimiterX\Del[1]\lparen\rparen{%
  \renewcommand\given{\SetSymbol[\delimsize]}
  #1
}

\def\ddefloop#1{\ifx\ddefloop#1\else\ddef{#1}\expandafter\ddefloop\fi}
\def\ddef#1{\expandafter\def\csname bf#1\endcsname{\ensuremath{\mathbf{#1}}}}
\ddefloop abcdefghijklmnopqrstuvwxyzABCDEFGHIJKLMNOPQRSTUVWXYZ\ddefloop
\def\ddef#1{\expandafter\def\csname bf#1\endcsname{\ensuremath{\boldsymbol{\csname #1\endcsname}}}}
\ddefloop {alpha}{beta}{gamma}{delta}{epsilon}{varepsilon}{zeta}{eta}{theta}{vartheta}{iota}{kappa}{lambda}{mu}{nu}{xi}{pi}{varpi}{rho}{varrho}{sigma}{varsigma}{tau}{upsilon}{phi}{varphi}{chi}{psi}{omega}{Gamma}{Delta}{Theta}{Lambda}{Xi}{Pi}{Sigma}{varSigma}{Upsilon}{Phi}{Psi}{Omega}{ell}\ddefloop
\def\ddef#1{\expandafter\def\csname cal#1\endcsname{\ensuremath{\mathcal{#1}}}}
\ddefloop ABCDEFGHIJKLMNOPQRSTUVWXYZ\ddefloop

\DeclareMathOperator\err{err}
\DeclarePairedDelimiterXPP\conderr[1]{\err}\lparen\rparen{}{%
  \renewcommand\given{\SetSymbol[\delimsize]}
  #1
}
\newcommand\groups{\calG}
\newcommand\hypotheses{\calH}
\newcommand\concepts{\calC_{\groups,\hypotheses}}
\newcommand\inputs{\calX}
\newcommand\bits{\set{0,1}}
\newcommand\signs{\set{-1,1}}

\newcommand\groupdim{d_{\groups}}
\newcommand\hypdim{d_{\hypotheses}}
\newcommand\Ghypdim{d_{\groups,\hypotheses}}
\newcommand\ghypdim{d_{g,\hypotheses}}

\DeclareMathOperator\poly{poly}

\newcommand\NP{\ensuremath{\mathsf{NP}}}

\usepackage{multicol}

\begin{document}

\maketitle

\begin{abstract}%
  The sample complexity of multi-group learning is shown to improve in the group-realizable setting over the agnostic setting, even when the family of groups is infinite so long as it has finite VC dimension.
  The improved sample complexity is obtained by empirical risk minimization over the class of group-realizable concepts, which itself could have infinite VC dimension.
  Implementing this approach is also shown to be computationally intractable, and an alternative approach is suggested based on improper learning.
\end{abstract}


\section{Introduction}

\emph{Multi-group learning}~\citep{rothblum2021multi} extends the basic framework of statistical learning to study the performance of predictive models within families of subpopulations.
Although multi-group learning can be cast as a special case of other learning frameworks that address subpopulation-level criteria~\citep[e.g.,][]{hebertjohnson2018multicalibration,kim2019multiaccuracy,dwork2021outcome,haghtalab2023unifying}, the goal of this work is to study a natural assumption within multi-group learning---\emph{group-realizability}---under which one might expect improved statistical efficiency and/or computational efficiency as compared to the general case.

In multi-group learning (for binary classification), subpopulations are specified by subsets of the input domain $\inputs$, and the learning objective concerns a family of (possibly overlapping) subpopulations $\groups \subseteq 2^{\inputs}$.
In particular, for a given collection of benchmark classifiers $\hypotheses \subseteq \signs^{\inputs}$ and an excess error rate bound $\epsilon \in \intoo{0,1}$, the learner seeks to construct a classifier $f \colon \inputs \to \signs$ using an i.i.d.~sample from a probability distribution $D$ over $\inputs \times \signs$ so that, with high probability,
\begin{equation}
  \conderr{f \given g} \leq \inf_{h \in \hypotheses} \conderr{h \given g} + \epsilon
  \quad \text{for each subpopulation $g \in \groups$} .
  \label{eq:multigroup}
\end{equation}
Above, $\conderr{f \given g} := \Pr_{(\bfx,\bfy) \sim D}\Sbr{ f(\bfx) \neq \bfy \given \bfx \in g }$ is the \emph{conditional error rate of a classifier $f$ given $g$} (defined whenever $g$ has non-zero mass under the marginal of $D$ over $\inputs$).
Importantly, each subpopulation $g$ may have a different optimal benchmark classifier $h_g^* \in \hypotheses$, and it is possible that these per-group optimal classifiers have disagreements $h_g^*(x) \neq h_{g'}^*(x)$ at points of intersection $x \in g \cap g'$.
Because of this, it is possible that no $f \in \hypotheses$ can satisfy \Cref{eq:multigroup}, and existing multi-group learning algorithms instead construct $f$ as an ensemble classifier involving functions from $\hypotheses$ and $\groups$.

\emph{Group-realizability} is the assumption on $(\groups,\hypotheses,D)$ that, for each $g \in \groups$, there is a benchmark classifier $h_g^* \in \groups$ such that $\conderr{h_g^* \given g} = 0$.
Under this assumption, the goal in multi-group learning is to construct a classifier $f$ such that $\conderr{f \given g} \leq \epsilon$ for each $g \in \groups$.
Group-realizability should be contrasted with the standard \emph{realizability} (or \emph{separability}) assumption on $(\hypotheses,D)$, i.e., the existence of $h^* \in \hypotheses$ such that $\Pr_{(\bfx,\bfy) \sim D}\sbr{ h^*(\bfx) \neq \bfy } = 0$.
Realizability implies group-realizability, but the reverse implication need not hold (unless, e.g., $\inputs \in \groups$).

A marked difference between the "realizable" and "agnostic" (i.e., non-realizable) settings in (single-group) statistical learning comes in the worst-case dependence on $\epsilon$ in the sample complexity: roughly $1/\epsilon$ versus $1/\epsilon^2$ (ignoring the dependence on $\hypotheses$).
A similar sample complexity difference also manifests in multi-group learning with and without the group-realizability assumption, as shown by \citet{tosh2021simple} for the case where the family of subpopulations $\groups$ is finite.
Moreover, the computational complexity is, roughly speaking, no worse than that of finding a consistent classifier in $\hypotheses$ for each subpopulation $g \in \groups$.
This is notable since, for some classes such as half-spaces, finding a consistent classifier can be done in polynomial-time, whereas finding a classifier in the class with approximately-minimal error rate on a sample may be computationally intractable~\citep{feldman2009agnostic,guruswami2009hardness}.

In this work, we show that the improved sample complexity (up to a logarithmic factor in $1/\epsilon$) for multi-group learning under group-realizability extends to the case where $\groups$ is infinite but has finite VC dimension.
To achieve this sample complexity guarantee, we introduce the class $\concepts \subseteq \signs^{\inputs}$ of \emph{group-realizable concepts}: the set of functions $c \colon \inputs \to \signs$ consistent with the group-realizability assumption. 
With this definition in hand, our main result follows simply from \emph{empirical risk minimization (ERM)} over $\concepts.$
In particular, there is no explicit or algorithmic regularization (in contrast to the algorithm of \citet{tosh2021simple}, which involves a form of aggregation).
Remarkably, it is possible for this class $\concepts$ to have infinite VC dimension, even when both $\groups$ and $\hypotheses$ have finite VC dimension.
ERM with infinite VC dimension classes is not generally an effective learning procedure in the standard statistical learning setup.
However, in our setup, the efficacy of ERM with such a class comes from the restricted sense in which the learned classifier is evaluated.
This is very simply captured using shattering coefficients.
The statistical efficiency of the $\concepts$ class is reminiscent of the near-optimality of ERM in realizable (single-group) statistical learning, where search over $\calH$ itself is sufficient for similar sample complexity guarantees.

Unfortunately, the rub with using ERM---which ultimately boils down to just finding a function $c \in \concepts$ consistent with the training data---is computational intractability: it generally involves solving an \NP-hard problem.
The hardness does not come from any potential intractability of optimizing over $\hypotheses$ or even enumerating $\groups$: our proof of hardness in \Cref{sec:computation} goes via a reduction to instances where $\groups$ has polynomial-size and optimization over $\hypotheses$ can be performed in polynomial-time.
Rather, the difficulty appears to come from even specifying a classifier in $\concepts$ when one can only refer to classifiers from $\hypotheses$ and subpopulations from $\groups$.
We show, in \Cref{sec:improper}, how "improper learning" can get around this intractability in certain cases.

\subsection*{Relation to prior works}

The literature on multi-group learning has focused primarily on the general (agnostic) setting, whether in the online setting~\citep{blum2020advancing,deng2024groupwise} or the batch setting~\citep{rothblum2021multi,tosh2021simple,rittler2023agnostic,deng2024multi}.
The work of \citet{tosh2021simple} provides a (batch) multi-group learning algorithm, called "Prepend", for families of subpopulations $\groups$ that may be infinite but have finite VC dimension.
Prepend---which is nearly identical to an algorithm proposed by \citet{globus2022algorithmic} in a somewhat related context---is efficient as long as it has access to an oracle for solving ERM-type optimization problems over $\hypotheses \times \groups$.
Such "oracle efficient" algorithms have a long history in many areas of learning theory~\citep[e.g.,][]{kalai2005efficient,kakade2005batch,dasgupta2007general,dudik2011efficient,dann2018oracle,foster2020beyond,haghtalab2022oracle,wang2022adaptive,garg2024oracle}, including the closely related subject of subgroup fairness~\citep{kearns2018preventing}.

An important drawback of Prepend, however, is that its sample complexity is suboptimal: thr dependence on $\epsilon$ is roughly $1/\epsilon^3$.
(Here, for simplicity, we are omitting the dependence on $\groups$, $\hypotheses$, and the smallest group probability mass.)
Prepend can be specialized to the group-realizable setting, but its sample complexity remains suboptimal: roughly $1/\epsilon^2$.
(Another algorithm of \citet{kim2019multiaccuracy} based on multi-accuracy also has a suboptimal $\Omega(1/\epsilon^2)$ sample complexity under group-realizability.)
As mentioned before, \citet{tosh2021simple} do provide an algorithm (different from Prepend) with near-optimal dependence on $\epsilon$ in the sample complexity (in both the agnostic case and under group-realizability) for the case where $\groups$ is finite.
The sample complexity of their algorithm is linear in $\log(\card{\groups})$; the algorithm also explicitly enumerates $\groups$, and hence the computational complexity may be exponential in the sample size.
No other prior works on multi-group learning address the potential sample complexity improvements in the group-realizable setting, including works that view multi-group learning as a special case of other learning frameworks such as
multi-calibration~\citep{hebertjohnson2018multicalibration},
outcome indistinguishability~\citep{dwork2021outcome,rothblum2021multi},
and multi-objective learning~\citep{haghtalab2023unifying}.
Consequently, sample complexity guarantees obtained by reducing to these general frameworks are $1/\epsilon^2$ or worse.

Interestingly, our identification of $\concepts$ as the correct object of study leads to our main sample complexity result via a simple analysis. The introduction of this class of functions $\concepts$ does not appear in prior literature. Instead, prior algorithms in multi-group learning such as those presented in \cite{tosh2021simple} and \cite{rothblum2021multi} as noted above consider aggregation procedures over $\hypotheses$ and $\groups$ instead of explicitly defining the class of interest to search over. 

The \NP-hardness of the computational problem encountered by ERM over $\concepts$ is reminiscent of the hardness of proper learning in PAC learning~\citep[e.g.,][]{pitt1988computational}, although as discussed above, the nature of the hardness appears to be conceptually different.

\section{Setting}

Throughout, $\groups \subseteq \bits^{\inputs} \cong 2^{\inputs}$ refers to a family of subpopulations (a.k.a.~groups), and $\hypotheses \subseteq \signs^{\inputs}$ refers to a collection of benchmark classifiers (a.k.a.~hypotheses).
It will be convenient to regard a group $g \in \groups$ both as a function $g \colon \inputs \to \bits$ and as a subset $g \subseteq \inputs$.
For a class of functions $\calF$ defined over a domain $\calZ$, the $k$-th shattering coefficient is the largest possible number of $k$-tuples realized by $\calF$ on $k$ points from $\calZ$:
\begin{equation*}
  \calS_{\calZ}(\calF,k) := \sup_{z_1, \dots, z_k \in \calZ} \card{ \Set{\del{f(z_1), \dots, f(z_k)} \given f \in \calF } }
  .
\end{equation*}

A central contribution of this paper is the introduction of the class of \emph{group-realizable concepts} $\concepts$, which we define as follows.

\begin{definition}
  \label{def:group_realizable}
  Fix any family of groups $\groups \subseteq \bits^{\inputs}$ and any hypothesis class $\hypotheses \subseteq \signs^{\inputs}$ on the same input domain $\inputs$.
  The set of \emph{group-realizable concepts with respect to $(\groups,\hypotheses)$}, denoted by $\concepts \subseteq \signs^{\inputs}$, is the set of all functions $c \colon \inputs \to \signs$ such that, for each $g \in \groups$, there exists $h \in \hypotheses$ satisfying $c(x) = h(x)$ for all $x \in g$.
  Additionally, for a probability distribution $D$ on $\inputs \times \signs$, we say $(\groups,\hypotheses,D)$ satisfies \emph{group-realizability} if there exists $c^* \in \concepts$ such that $\Pr_{(\bfx,\bfy) \sim D}\sbr{ c^*(\bfx) = \bfy } = 1$.
\end{definition}

In this work, we only consider distributions $D$ over $\inputs \times \signs$ such that $(\groups,\hypotheses,D)$ satisfies group-realizability.
So we can equivalently specify $D$ by its marginal distribution $P$ over $\inputs$, and a group-realizable concept $c^* \in \concepts$.
The conditional error rate of a classifier $f \colon \inputs \to \signs$ given $g$ can therefore be written as
\begin{equation*}
  \conderr{f \given g} = \Pr_{\bfx \sim P}\Sbr{ f(\bfx) \neq c^*(\bfx) \given \bfx \in g } .
\end{equation*}
Our concern is the \emph{distribution-free} setting in the sense that we are interested in guarantees that hold for worst-case choices of $P$.
This mirrors the usual sense in which standard PAC learning is regarded as distribution-free even under realizability~\citep{valiant1984theory}.
Our aim is to achieve complexity guarantees that are comparable to those achievable in the standard setting when $\groups = \set{ \inputs }$. Formally, our goal is to furnish a classifier $f \colon \inputs \to \signs$ whose conditional-error rate on each group is small, captured by the following definition.

\begin{definition}
For any $(\groups, \hypotheses, D)$ satisfying group-realizability (Definition \ref{def:group_realizable}), a classifier $f \colon \inputs \to \signs$ achieves \emph{group-realizable multi-group learning} if, for all $(\epsilon, \delta) \in (0, 1)$,
\begin{equation*}
    \conderr{f \given g} \leq \epsilon \quad \text{for all } g \in \groups
\end{equation*}
with probability $1 - \delta$ over the i.i.d.~training examples $(\bfx_1, \dots, \bfx_n) \sim P^n.$
\end{definition}

Specifying a group-realizable concept $c \in \concepts$ typically begins by specifying an assignment $\groups \ni g \mapsto h_g \in \hypotheses$ of hypotheses to groups.
Note that an assignment of hypotheses to groups for which there are "disagreements" (e.g., $h_g(x) \neq h_{g'}(x)$ for some $x \in g \cap g'$) does not directly yield a well-defined classifier.
In this work, we consider the following options to get a valid classifier:
\begin{enumerate}
  \item Ensure the chosen $h_g$'s have no disagreements, i.e., whenever any two groups $g, g' \in \groups$ have a non-empty intersection, we have $h_g(x) = h_{g'}(x)$ for all $x \in g \cap g'$, so the assignment corresponds to some $c \in \concepts$.
    The sample and computational complexities of this approach are explored in \Cref{sec:sample,sec:computation}.

  \item Reconcile disagreements among the $h_g$'s; a data-driven approach is given in \Cref{sec:improper}.
    This approach is not guaranteed to yield a concept from $\concepts$.

\end{enumerate}

Even when $\groups$ and $\hypotheses$ have finite VC dimension, it is possible for $\concepts$ to have infinite VC dimension, as the following \namecref{prop:grouprealizableexample} shows.

\begin{proposition}
  \label{prop:grouprealizableexample}
  There exists $\groups \subseteq \bits^{\inputs}$ and $\hypotheses \subseteq \signs^{\inputs}$ on the same input domain $\inputs$, each with finite VC dimension, such that $\concepts$ has infinite VC dimension.
\end{proposition}
\begin{proof}
  Let $\inputs$ be any set of infinite cardinality (e.g., the integers).
  Let $\hypotheses$ consist of just the "constant $1$" function and the "constant $-1$" function.
  This class is finite (and has VC dimension $1$).
  Let $\groups = \Set{ \set{x} \given x \in \inputs }$ be the family of singleton sets.
  This class has VC dimension $1$.
  The groups in $\groups$ are disjoint, so $\concepts$ contains all $\signs$-valued functions on $\inputs$.
  Therefore, finite subsets of $\inputs$ of all sizes are shattered by $\concepts$,  which means that $\concepts$ has infinite VC dimension.
\end{proof}

The set of group-realizable concepts $\concepts$ is not an appropriate class to use in the agnostic setting, where $(\groups,\hypotheses,D)$ may fail to satisfy group-realizability.
For instance, suppose $\groups = \set{ g_1, g_2 }$ with $g_1 \cap g_2 \neq \emptyset$, and $\hypotheses = \set{ x \mapsto -1, x \mapsto 1 }$ contains just the constant $-1$ and constant $1$ hypotheses.
Then $\concepts = \set{ x \mapsto -1, x \mapsto 1}$ as well.
However, consider $(\bfx,\bfy) \sim D$ with $P(g_1 \setminus g_2) = P(g_1 \cap g_2) = P(g_2 \setminus g_1) = 1/3$ (where $P$ is the marginal distribution of $\bfx$),
$\Pr\Sbr{ \bfy = 1 \given \bfx \in g_1 \setminus g_2 } = 1/2$, 
$\Pr\Sbr{ \bfy = 1 \given \bfx \in g_1 \cap g_2 } = 2/3$, 
$\Pr\Sbr{ \bfy = 1 \given \bfx \in g_2 \setminus g_1 } = 0$.
In this case, the best constant predictor for $g_1$ is $x \mapsto 1$, but the best constant predictor for $g_2$ is $x \mapsto -1$.

\section{Sample complexity}
\label{sec:sample}

Our main sample complexity result is a consequence of the following \namecref{thm:sample}.

\begin{theorem}
  \label{thm:sample}
  Let $P$ be a probability distribution on $\inputs$, let $\groups \subseteq \bits^{\inputs}$ be any family of groups on $\inputs$, and let $\hypotheses \subseteq \signs^{\inputs}$ be any hypothesis class on $\inputs$.
  Fix any $c^* \in \concepts$, and let $(\bfx_1,\dotsc,\bfx_n) \sim P^n$ and $\bfS = ((\bfx_i,c^*(\bfx_i)))_{i \in [n]}$.
  \begin{enumerate}
    \item For any $g \in \groups$ and any $\delta \in \intoo{0,1}$, with probability at least $1-\delta$, every $c \in \concepts$ consistent with $\bfS$ has
      \begin{equation}
        \Pr_{\bfx \sim P}\sbr{ c(\bfx) \neq c^*(\bfx) \wedge \bfx \in g }
        \leq \frac{4\del*{ \log\binom{2n}{\leq \ghypdim} + \log\del*{\tfrac4\delta} }}{n} .
        \label{eq:foreachbound}
      \end{equation}
      Above, $\ghypdim$ is the VC dimension of $\hypotheses$ restricted to $g \subseteq \inputs$.

    \item For any $\delta \in \intoo{0,1}$, with probability at least $1-\delta$, every $c \in \concepts$ consistent with $\bfS$ has
      \begin{equation}
        \Pr_{\bfx \sim P}\sbr{ c(\bfx) \neq c^*(\bfx) \wedge \bfx \in g }
        \leq \frac{4\del*{ \log\binom{2n}{\leq \groupdim} + \log\binom{2n}{\leq \sup_{g \in \groups} \ghypdim} + \log\del*{\tfrac4\delta} }}{n}
        \quad \forall g \in \groups .
        \label{eq:forallbound}
      \end{equation}
      Above, $\ghypdim$ is the VC dimension of $\hypotheses$ restricted to $g \subseteq \inputs$, and $\groupdim$ is the VC dimension of $\groups$.
  \end{enumerate}
\end{theorem}

Although the guarantees of \Cref{thm:sample} are not stated in terms of the conditional error rates $\conderr{c \given g}$, such error rates can be gotten by dividing by $P(g) = \Pr_{\bfx \sim P}\sbr{ \bfx \in g}$.
After doing so, the right-hand sides in \Cref{eq:foreachbound,eq:forallbound} have a denominator of $n P(g)$, which can be interpreted as the expectation of number of training examples $N_g$ from $g$.
Therefore, the conditional error rate decreases roughly as $1/N_g$, which should be contrasted to the $1/\sqrt{N_g}$ rates in the agnostic case~\citep{tosh2021simple}.

Furthermore, to obtain sample size requirements for multi-group learning, we can simply "solve for $n$" to make the right-hand side equal to (or bounded above by) the target excess error rate bound $\epsilon$.
For instance, from \Cref{eq:forallbound}, we derive the sample size requirement
\begin{equation}
  n \geq C \cdot \frac{(\Ghypdim + \groupdim) \log(1/\gamma\epsilon) + \log(1/\delta)}{\gamma\epsilon}
  ,
  \label{eq:sample_complexity_vc}
\end{equation}
where $C>0$ is some absolute constant,
$\Ghypdim := \sup_{g \in \groups} \ghypdim$,
and $\gamma$ is a lower-bound on $\Pr_{\bfx \sim P}\sbr{ \bfx \in g}$ that holds for all $g \in \groups$.\footnote{%
  The dependence on $\gamma$ can be alleviated if one is willing to consider non-uniform error bounds that scale (inversely) with $P(g)$ for group $g$; see discussion of~\citet{tosh2021simple} on this matter.
  This aspect is common to all multi-group learning algorithms.%
}
Note that $\Ghypdim$ is always at most the VC dimension $\hypdim$ of $\hypotheses$, and is equal to $\hypdim$ whenever $\inputs \in \groups$.
For comparison, the sample size requirement of the algorithm of \citet{tosh2021simple} in the group-realizable setting is
\begin{equation}
  n \geq C \cdot \frac{\hypdim \log(1/\gamma\epsilon) + \log(\card{\groups}) + \log(1/\delta)}{\gamma\epsilon}
  .
  \label{eq:sample_complexity_cardinality}
\end{equation}
(See \Cref{sec:improper} for more discussion.)
In both cases, when $\groups = \set{ \inputs }$, the sample size requirement reduces to the usual sample complexity for ERM in the realizable setting (which is within a factor of $\log(1/\epsilon)$ from optimal). Because multi-group learning generalizes classical (single-group) realizable learning when $\calG = \set{\inputs}$, the classical lower bound of $\Omega(\hypdim/\epsilon)$ of \cite{blumer1989learnability} demonstrates that our dependence on $\epsilon$ is tight up to the $\log(1/\epsilon)$ factor.

The difference between \Cref{eq:sample_complexity_vc} and \Cref{eq:sample_complexity_cardinality} manifests for classes where $\log(\card{\groups}) \gg \groupdim \log(1/\gamma\epsilon)$.
Perhaps more interesting, however, is that \Cref{eq:sample_complexity_vc} is achieved via ERM over $\concepts$ without any explicit regularization, which stands in contrast to the algorithm of \citet{tosh2021simple}, which explicitly uses aggregation.

In the first part of \Cref{thm:sample} (specifically \Cref{eq:foreachbound}), we see that there is no dependence on $\groups$ whatsoever.
The guarantee in \Cref{eq:foreachbound} holds with probability $1-\delta$ for each group $g \in \groups$, but not for all groups simultaneously.
This is relevant in cases where the downstream evaluation is ultimately based only on performance in a single group, but the identity of that group is not known at the training time.

The proof of (each part of) \Cref{thm:sample} is a simple consequence of the following \namecref{lem:uniform}.
In \Cref{lem:uniform}, we use the standard shorthand $P f := \mathbb{E}_{\bfx \sim P}[f(\bfx)]$ and $P_n f := \frac{1}{n} \sum_{i = 1}^n f(\bfx_i)$ for random variables and $(\bfx_1, \dots, \bfx_n) \sim P^n$.

\begin{lemma}[\citealp*{vapnik1971uniform}]
  \label{lem:uniform}
  Let $\calF$ be a family of measurable functions $f \colon \calZ \to \bits$, and let $\delta \in \intoo{0,1}$.
  Let $\alpha_n = (4/n) \ln (4\calS_{\calZ}(\calF, 2n)/\delta)$, where $\calS_{\calZ}(\calF,k)$ is the $k$-th shattering coefficient for the class $\calF$.
  Let $P$ be any distribution over $\calZ$, and let $P_n$ be the empirical distribution on an i.i.d.~sample of size $n$ from $P$.
  With probability at least $1 - \delta$, for all $f \in \calF$:
  \begin{equation*}
    \frac{P f - P_n f}{\sqrt{P f}} \leq \sqrt{\alpha_n}
  \end{equation*}
\end{lemma}

\begin{proofof}{\Cref{thm:sample}}
  For the first part of the claim, we let
  \begin{equation*}
    \calF := \Set{ x \mapsto g(x) (c \triangle c^*)(x) \given c \in \concepts } ,
  \end{equation*}
  where $(c\triangle c^*)(x) = \ind{c(x) \neq c^*(x)}$.
  Consider any $2n$ points $X := (x_1,\dotsc,x_{2n})$ from $\inputs$, so
  \begin{equation*}
    \calF \rvert_X
    = \Set*{ (f(x_i))_{i \in [2n]} \given f \in \calF }
    = \Set*{ (g(x_i)(c\triangle c^*)(x_i))_{i \in [2n]} \given c \in \concepts } .
  \end{equation*}
  By definition of $\concepts$, each vector $v \in \calF \rvert_X$ is equivalent to $(g(x_i)(h\triangle c^*)(x_i))_{i \in [2n]}$ for some $h \in \hypotheses$.
  Therefore,
  \begin{equation*}
    \card{\calF \rvert_X} \leq \calS_g(\hypotheses,2n)
    \leq \binom{2n}{\leq \ghypdim}
  \end{equation*}
  by Sauer's lemma.
  Since this holds for all choices of $x_1,\dotsc,x_{2n} \in \inputs$, the above bound also holds for $\calS_{\inputs}(\calF,2n)$. To finish the proof, we use a standard manipulation on \Cref{lem:uniform}. For any non-negative numbers $A, B, C$, we have that $A \leq B + C \sqrt{A}$ implies $A \leq B + C^2 + \sqrt{B}C$, which results in
  \begin{equation*}
    P f \leq P_n f + 2 \sqrt{P_n f \frac{\log \calS_{\inputs}(\calF, 2n) + \log(4/\delta)}{n}} + 4 \frac{\log \calS_{\inputs}(\calF, 2n) + \log(4/\delta)}{n}
  \end{equation*}
  for all $f \in \calF$. We assume that $c \in \concepts$ is consistent with $\bfS$, so noting that $f(x) = g(x)(c\triangle c^*)(x)$ implies that $P_n f = 0.$ Therefore,
  \begin{equation*}
    P f \leq 4 \frac{\log \calS_{\inputs}(\calF, 2n) + \log(4/\delta)}{n}
  \end{equation*}
  and the result follows from plugging in the bound on $\calS_g(\calH, 2n)$ established above.

  For the second claim, we just change the class $\calF$ to
  \begin{equation*}
    \calF := \Set{ x \mapsto g(x) (c \triangle c^*)(x) \given g \in \groups, c \in \concepts }.
  \end{equation*}
  Again, for any $2n$ points $X := (x_1,\dotsc,x_{2n})$ from $\inputs$,
  \begin{equation*}
    \calF \rvert_X
    = \Set*{ (g(x_i)(c\triangle c^*)(x_i))_{i \in [2n]} \given g \in \groups, c \in \concepts } .
  \end{equation*}
  By definition of $\concepts$, each vector $v \in \calF \rvert_X$ is equivalent to $(g(x_i)(h\triangle c^*)(x_i))_{i \in [2n]}$ for some $g \in \groups$ and some $h \in \hypotheses$.
  Each such vector is obtained (via component-wise product) from a vector of the form $(g(x_i))_{i \in [2n]}$ for some $g \in \groups$, and a vector of the form $(h(x_i))_{i \in [2n], x_i \in g}$ for some $h \in \hypotheses$.
  Therefore,
  \begin{equation*}
    \card{\calF \rvert_X}
    \leq \calS_{\inputs}(\groups,2n) \cdot \sup_{g \in \groups} \calS_g(\hypotheses,2n)
    \leq \binom{2n}{\leq \groupdim} \cdot \sup_{g \in \groups} \binom{2n}{\leq \ghypdim} .
  \end{equation*}
  The rest is the same as the proof for the first part.
\end{proofof}

In the proof of \Cref{thm:sample}, we see that the full "richness" of $\concepts$ is never encountered because the "mistake behaviors" that are counted are always restricted to individual groups.
On any group $g \in \groups$, the behavior of a group-realizable concept $c \in \concepts$ is determined by some hypothesis $h_g \in \hypotheses$, so the number of behaviors is limited if $\groups$ and $\hypotheses$ have finite VC dimension.
This kind of argument is reminiscent of the concept of computational indistinguishability from cryptography, which has recently been adopted in the context of learning criteria such as multi-calibration~\citep{hebertjohnson2018multicalibration} and outcome indistinguishability~\citep{dwork2021outcome}.
In our case, the "adversary" that tries to find a fault in a learner's classifier makes an appearance only in the analysis, rather than explicitly in an algorithm that, say, simulates game dynamics~\citep[e.g.,][]{haghtalab2023unifying}.

\section{Computational complexity}
\label{sec:computation}

Even though ERM with $\concepts$ has good sample complexity in the group-realizable setting, it may be difficult to implement, even in cases where ERM is easy to implement with $\hypotheses$ itself (and where the size of the family of groups is small).
Specifically, we show that given a succinct description of $\groups$ and $\hypotheses$, along with a labeled dataset $S$, it is \NP-hard to decide if there is a concept in $\concepts$ that is consistent with $S$, even if provided an oracle that returns a consistent hypothesis from $\hypotheses$ for any given labeled dataset (should one exist).

We give a reduction from the \NP-complete decision problem \textsf{ONE-IN-THREE 3SAT}~\citep{schaefer1978complexity,garey1979computers}:
Given a $3$-CNF formula $\phi$, decide if there is a truth assignment to the variables such that each clause in $\phi$ has exactly one literal that evaluates to "true".

\paragraph{Overview of the reduction.}

Given a $3$-CNF formula, we construct a group family with one group per clause, and a hypothesis class where hypotheses correspond to truth assignments.
The hypothesis class is a disjoint union of clause-specific (i.e., group-specific) hypothesis classes, where the hypotheses for a particular clause are those that correspond to all possible truth assignments that make \emph{exactly} one literal in the clause evaluate to "true".
Recall that to specify a concept in $\concepts$, we should choose a hypothesis in $\hypotheses$ for each $g \in \groups$.
Ensuring that the choice of hypothesis for $g \in \groups$ comes from the $g$-specific hypothesis class is handled by introducing additional points in the input domain/groups, and extending the behavior of the hypotheses and constructing the labeled dataset $S$ in a natural way.

\paragraph{Notation.}

For a literal $l$, let $v(l)$ denote the variable $x$ such that $l \in \set{x,\neg x}$, and let
\begin{equation*}
  y(l) := \begin{cases}
    +1 & \text{if $l = v(l)$} \\
    -1 & \text{if $l = \neg v(l)$}
  \end{cases}
\end{equation*}
denote the "polarity label" of that literal.

\paragraph{Reduction.}

Let $\phi$ be a $3$-CNF formula over variables $x_1,\dotsc,x_n$ with clauses $C_1,\dotsc,C_m$.
Each clause $C_i$ is the disjunction of three literals, $C_i = l_i^1 \vee l_i^2 \vee l_i^3$.

We use $\phi$ to define the group family $\groups$, hypothesis class $\hypotheses$, and labeled dataset $S$ as follows.
\begin{enumerate}
  \item The input domain $\inputs$ consists of $n+m$ points, one per variable and clause.
    We use the same names ($x_1,\dotsc,x_n,C_1,\dotsc,C_m$) for these points as the variables and clauses.

  \item The group family is $\groups := \set{g_1,\dotsc,g_m}$, with one group per clause.

  \item For each clause $C_i = l_i^1 \vee l_i^2 \vee l_i^3$, define $g_i := \set{v(l_i^1), v(l_i^2), v(l_i^3), C_i}$.
    Also, define $\hypotheses_i := \hypotheses_i^1 \cup \hypotheses_i^2 \cup \hypotheses_i^3$, where for each $t \in \set{1,2,3}$, $\hypotheses_i^t$ contains all possible hypotheses $h \colon \inputs \to \signs$ satisfying the following:
    \begin{multicols}{2}
      \begin{enumerate}
        \item $h(v(l_i^t)) = y(l_i^t)$;
        \item $h(v(l_i^s)) = - y(l_i^s)$ for $s \neq t$;
        \item $h(C_i) = +1$;
        \item $h(C_j) = -1$ for all $j \neq i$.
      \end{enumerate}
    \end{multicols}
  \item The overall hypothesis class is $\hypotheses := \hypotheses_1 \cup \dots \cup \hypotheses_m$.

  \item The labeled dataset is $S := ((C_1,1),\dotsc,(C_m,1))$.

\end{enumerate}
Soundness and completeness of the reduction are immediate,
and a description of $(\groups, \hypotheses, S)$ can be produced from $\phi$ in $\poly(n)$ time.

\paragraph{Efficiency of searching for a consistent hypothesis.}

We claim that it is easy to search for a consistent hypothesis in the hypothesis class $\hypotheses$ constructed by the above reduction.
(This holds even for the hypothesis class derived from a CNF formula with no clause width restriction.)

Recall that $\hypotheses$ is the (disjoint) union of $\hypotheses_1,\dotsc,\hypotheses_m$.
We first show how to search for a hypothesis in $\hypotheses_j$ (for a fixed $j \in \set{1,\dotsc,m}$) that is consistent with any set of labeled example $S' \subseteq (\inputs \setminus \set{C_1,\dotsc,C_m}) \times \signs$.
Define the literals $\ell_{(x,-1)} = \neg x$ and $\ell_{(x,+1)} = x$, and define the term $T$ to be the conjunction of literals $\ell_{(x,y)}$ for $(x,y) \in S'$.
Then, there is a hypothesis $h \in \hypotheses_j$ consistent with $S'$ if and only if there is a truth assignment such that: $T$ is satisfied, and $C_j$ has exactly one literal that evaluates to "true".
To search for such a hypothesis: construct a partial truth assignment that satisfies $T$; if such a partial assignment exists, then check if it can be extended to satisfy exactly one of the literals in $C_j$.
This can be performed with a linear scan over $T$ and $C_j$.

To see that it is easy to search for a hypothesis in $\hypotheses$ that is consistent with a collection of labeled examples $S \subseteq \inputs \times \signs$, it suffices to explain how to determine which $\hypotheses_j$ to search: this ultimately hinges upon which examples of the form $(C_i,y)$ are in $S$.
If there are no examples of the form $(C_i,y)$ in $S$, then the search is unrestricted.
If $(C_i,+1) \in S$, then the search excludes all $\hypotheses_j$ for $j \neq i$.
If $(C_i,-1) \in S$, then the search excludes $\hypotheses_i$.
Of course, it is possible that all hypotheses are ultimately excluded, in which case there is clearly no consistent hypothesis.

\paragraph{Implications.}

The reduction above shows that it is $\NP$-hard to find a $c \in \concepts$ consistent with a collection of labeled examples, even if provided an oracle for finding hypotheses in $\hypotheses$ consistent with any given $\poly(N)$-many labeled examples from $\inputs \times \signs$, and even if $\card{\groups} = \poly(N)$, where $N$ represents the dimension or description length of inputs, hypotheses, and groups.

\section{Improper multi-group learning under group-realizability}
\label{sec:improper}

This computational intractability from \Cref{sec:computation} is specific to "proper" multi-group learning under group-realizability; it only applies to situations where one seeks to find a classifier from $\concepts$.
As previously mentioned, the intractability can be subverted with improper learning, at least in some situations (including the specific scenario from the reduction in \Cref{sec:computation}) by using the approach of \citet{tosh2021simple}:
\begin{itemize}
  \item Let $\bfS$ be $n$ i.i.d.~labeled examples from the distribution $D = (P,c^*)$.

  \item For each $g \in \groups$, let $\hat h_g \in \hypotheses$ be any hypothesis consistent with the first $n/2$ examples in $\bfS$.

  \item Run \citeauthor{tosh2021simple}'s Algorithm~2 with group family $\groups$, hypotheses $\Set{ \hat h_g \given g \in \groups }$, and learning rate $\eta=1/2$ on the last $n/2$ examples in $\bfS$, to obtain the final classifier $f$.
    This algorithm is a specific instantiation of an online learning algortihm of \citet{blum2007external} combined with online-to-batch conversion.

\end{itemize}

The classifier $f$ output in the last step is randomized, although a deterministic classifier can be easily obtained with a simple modification to their algorithm (specifically, replacing the algorithm of~\citet{blum2007external} with a suitable variant of \citet{littlestone1994weighted}'s Weighted Majority).
In either case, the salient point here is that $f$ is not selected from $\concepts$; there is no attempt to ensure that hypotheses assigned to groups (from the second step above) "agree" on regions of intersection, so the intractability results from \Cref{sec:computation} are not applicable.
Instead, these hypotheses are combined in an ensemble classifier using online learning and online-to-batch conversion.

As mentioned in \Cref{sec:sample}, the sample size requirement of this algorithm is given in \Cref{eq:sample_complexity_cardinality}, which is comparable to that of ERM over $\concepts$ when $\log(\card{\groups}) \lesssim \groupdim \log(1/\epsilon)$.
This algorithm runs in polynomial-time whenever $\groups$ has polynomial cardinality and finding consistent hypotheses from $\hypotheses$ can be done in polynomial-time.
These aforementioned conditions hold for the $(\groups,\hypotheses)$ constructed in the reduction from \Cref{sec:computation}.
It is in these scenarios that computational intractability is subverted by improper learning.

\section{Discussion and future directions}

The statistical efficiency of ERM over the rich class $\concepts$ is a remarkable phenomenon, and it seems worthy of further investigation in other settings, including general (agnostic) multi-group learning with a suitable relaxation of $\concepts$.
In our setting, ERM over $\concepts$ is computationally intractable, but an ensemble method sometimes gets around the intractability without a statistical cost.
This is reminiscent of convex relaxation and other ways improper learning offer computational speed-ups.
It would be interesting to understand if these other approaches are also applicable in our setting.

A problem left open is to find a general oracle-efficient multi-group learning algorithm that achieves the sample complexity from \Cref{eq:sample_complexity_vc} in the group-realizable setting. 
One possible line of attack is to leverage recent progress on oracle-efficient algorithms in related settings~\citep{deng2024groupwise,okoroafor2025near}.

\acks{%
  We acknowledge support from the ONR under grant N00014-24-1-2700.
  SD also acknowledges the support of the Avanessians Doctoral Fellowship for Engineering Thought Leaders and Innovators in Data Science.
  This work grew out of discussions during the ``Modern Paradigms in Generalization'' program at the Simons Institute for the Theory of Computing, Berkeley in 2024.%
}

\bibliography{bib}

\end{document}